\documentclass[runningheads]{llncs}
\usepackage[T1]{fontenc}
\usepackage{graphicx}
\usepackage{multirow}
\usepackage{booktabs}
\usepackage{algorithm}
\usepackage{algpseudocode}
\usepackage[misc]{ifsym}
\usepackage{arydshln}
\usepackage{color}
\usepackage{xfrac}
\usepackage{colortbl}
\usepackage[pagebackref=true,breaklinks=true,colorlinks=true,bookmarks=false,citecolor=blue,urlcolor=blue,linkcolor=blue]{hyperref}

\usepackage{booktabs}
\usepackage{multirow}
\usepackage[table,xcdraw]{xcolor}
\newcommand{\Tau}{\mathcal{T}}
\usepackage{amssymb}
\usepackage{amsfonts}
\usepackage{amsmath} 
\usepackage{bm}

\newcommand{\E}{\mathbb{E}}

\begin{document}
\title{Client-Level Differential Privacy via Adaptive Intermediary in Federated Medical Imaging}
\titlerunning{Client-Level Differential Privacy  in Federated Medical Imaging}
\author{
Meiqui Jiang\inst{1} 
\and
Yuan Zhong\inst{1} 
\and
Anjie Le\inst{1}
Xiaoxiao Li\inst{2}
\and 
Qi Dou\textsuperscript{1(\Letter)}
}

\authorrunning{M. Jiang et al.}

\institute{
Dept. of Computer Science and Engineering, The Chinese University of Hong Kong
\and
Dept. of Electrical and Computer Engineering, The University of British Columbia
}
\maketitle              
\begin{abstract}
Despite recent progress in enhancing the privacy of federated learning (FL) via differential privacy (DP), the trade-off of DP between privacy protection and performance is still underexplored for real-world medical scenario. In this paper, we propose to optimize the trade-off under the context of client-level DP, which focuses on privacy during communications. However, FL for medical imaging involves typically much fewer participants (hospitals) than other domains (e.g., mobile devices), thus ensuring clients be differentially private is much more challenging. To tackle this problem, we propose an adaptive intermediary strategy to improve performance without harming privacy. Specifically, we theoretically find splitting clients into sub-clients, which serve as intermediaries between hospitals and the server, can mitigate the noises introduced by DP without harming privacy. Our proposed approach is empirically evaluated on both classification and segmentation tasks using two public datasets, and its effectiveness is demonstrated with significant performance improvements and comprehensive analytical studies.
Code is available at: \url{https://github.com/med-air/Client-DP-FL}.

\end{abstract}
\section{Introduction}

Differential privacy (DP) has emerged as a promising technique to safeguard the privacy of sensitive data in federated learning (FL)~\cite{adnan2022federated,kairouz2019advances,kaissis2021end,rieke2020future,ziller2021differentially}, offering privacy guarantees in a mathematical format~\cite{dwork2006calibrating,mironov2017renyi,zheng2021federated}. However, introducing noise to ensure DP often comes at the cost of performance. Some recent studies have noticed that the noise added to the gradient impedes optimization~\cite{de2022unlocking,kim2021federated,papernot2021tempered}. For critical medical applications requiring low error tolerance, such performance degradation makes the rigorous privacy guarantee diminish~\cite{dayan2021federated,liu2022medical}. Therefore, it is imperative to maintain high performance while enhancing privacy, i.e., optimizing the privacy-performance trade-off. Unfortunately, despite its significance, such trade-off optimization in FL has not been sufficiently investigated to date.

Several studies have examined the trade-off in the centralized scenario. For instance, Li et al.~\cite{li2022private} proposed enhancing utility by leveraging public data or data statistics to estimate gradient geometry. Amid et al.~\cite{amid2022public} utilized the loss on public data as a mirror map to improve performance. Li et al.~\cite{li2023differentially} suggested constructing less noisy preconditioners using historical gradients. In contrast to these studies, we concentrate on promoting the trade-off in FL, where public dataset is limited and sharing side information may not be feasible~\cite{kairouz2019advances,rieke2020future}. 
Specifically, we aim to ensure that clients are differentially private. Our objective is not to protect a single data point, but rather to achieve that a learned model does not reveal whether a client participated in decentralized training. This ensures that a client's entire dataset is safeguarded against differential attacks from third parties. 
This is particularly crucial in medical imaging, where sensitive patient information is typically kept within each hospital. 
Nevertheless, in medical imaging, the number of participants (silos) is usually much smaller than in other domains, such as mobile devices~\cite{kairouz2019advances}. This cross-silo situation necessitates adding a considerable amount of noise to protect client privacy, making the optimization of the trade-off uniquely challenging \cite{liu2022on}.

To improve the trade-off of privacy protection and performance, the key point is to mitigate the noise added to the client during gradient updates. Our idea is inspired by the observation in DP-FedAvg~\cite{brendan2018learning}, which suggests that the utility of DP can be improved by utilizing a sufficiently large dataset with numerous users. Through an analysis of the DP accountant, we identified that the noise is closely related to the gradient clip bound and the number of participants. In this regard, we propose to split the original client into disjoint sub-clients, which act as intermediaries for exchanging information between the hospital and the server. This strategy increases the number of client updates against queries, thereby consequently reducing the magnitude of noise. However, finding an optimal splitting is not straightforward due to the non-identical nature of data samples. Splitting a client into more sub-clients may increase the diversity of FL training, which can adversely harm the final performance. Thus, there is a trade-off between noise level and training diversity. Our objective is to explore the relationships among clients, noise effects, and training diversities to identify a balance point that maximizes the trade-off between privacy and performance.

In this paper, we present a novel adaptive intermediary method to optimize the privacy-performance trade-off. Our approach is based on the interplay relationships among noise levels, training diversities, and the number of clients. Specifically, we observe a reciprocal correlation between the noise level and the number of intermediaries, as well as a linear correlation between the training diversity and the intermediary number. To determine the optimal number of intermediaries, we introduce a new term called intermediary ratio, which quantifies the ratio of noise level and training diversity. 
Our theoretical analysis demonstrates that splitting the original clients into more intermediaries achieves DP with the same privacy budget and DP failure probability. 
Furthermore, we show that when sample-level DP and client-level DP have equivalent noise levels, the variance of the difference between noisy and original model diverges exponentially with more training steps, leading to poor performance.
We evaluate our method on both classification and segmentation tasks, including the intracranial hemorrhage diagnosis with 25,000 CT slices, and the prostate MRI segmentation with heterogeneous data from different hospitals. Our method consistently outperforms various DP optimization methods on both tasks and can serve as a lightweight add-on with good compatibility. In addition, we conduct comprehensive analytical studies to demonstrate the effectiveness of our method.

\section{Method}
\subsection{Preliminaries}
In this work, we consider client-level differential privacy. We first introduce the definition of DP as follows:

\begin{figure*}[t]
\centering
\includegraphics[width=0.99\columnwidth]{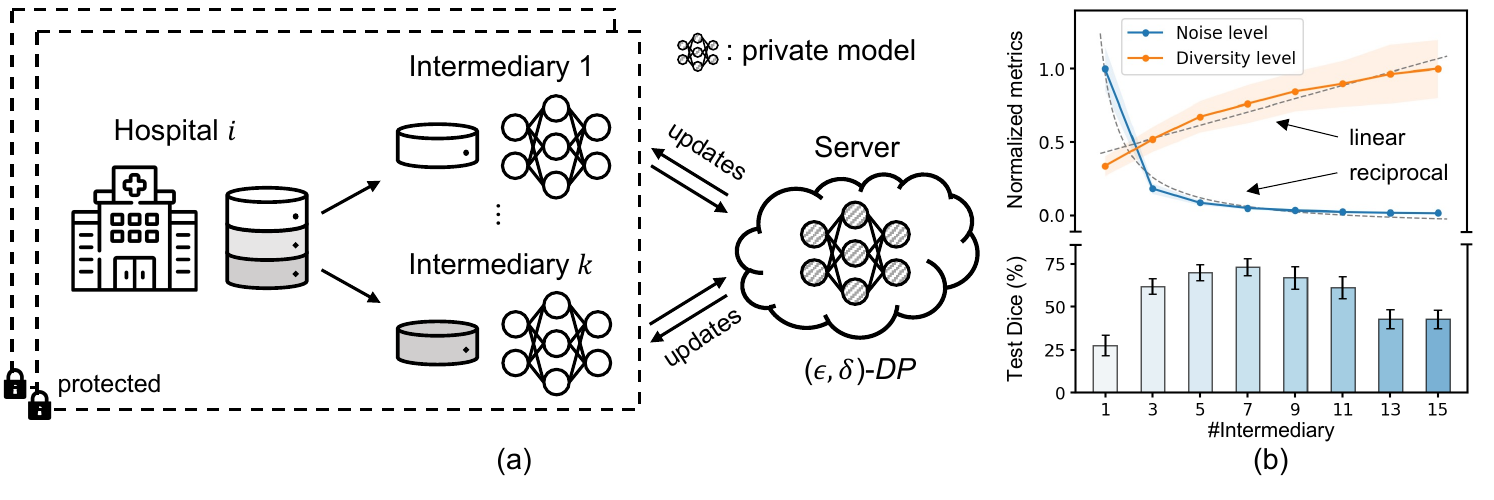}
\caption{(a) Overview of our intermediary strategy, which protects participating hospitals with superior privacy-performance trade-offs. The server aggregates local models from non-overlapping intermediaries with DP guarantees. (b) Continuously splitting intermediaries may not continuously improve performance, as it reduces gradient noises in a reciprocal manner, but also increases gradient diversity or heterogeneity linearly.}

\label{fig:method}
\end{figure*}

\begin{definition}{($(\epsilon, \delta)$-Differential Privacy~\cite{dwork2006calibrating,dwork2014algorithmic})}
For a randomized learning mechanism $M$: $\mathcal{X} \rightarrow \mathcal{R}$, where $\mathcal{X}$ is the collection of datasets it can be trained on, and $\mathcal{Y}$ is the collection of model it can generate, it is $(\epsilon, \delta)$-DP if:
$$(\forall S \subseteq \mathcal{R})(\forall D,D'\in \mathcal{X}, D\sim D') 
\operatorname{Pr}[M(D\-) \in S] \leq \exp 
(\varepsilon) \cdot 
\operatorname{Pr}\left[M\left(D^{\prime}\right) \in 
S\right]+\delta,$$
\end{definition}
where $\epsilon$ denotes the privacy budget, and $\delta$ represents the probability that $\epsilon$-DP fails in this mechanism. Note that the smaller the $\epsilon$ value is, the more private the mechanism is. 
Our aim of applying DP is to protect the collection of ``datasets'' $\mathcal{X}$, which are client model updates in every communication round in the context of FL.
The protection can be done by incorporating a DP-preserving randomization mechanism into the learning process. One commonly used method is the Gaussian mechanism, which involves bounding the contribution ($l_2$-norm) of each client update followed by adding Gaussian noise proportional to that bound onto the aggregate \cite{brendan2018learning}. Specifically, 
suppose there are $N$ clients, denote the gradients of each client as $\Delta_i$, the server model $\theta_{t+1}$ at round $t+1$ is updated by adding the Gaussian mechanism approximating the sum of updates as follows:

\begin{equation}
\label{eq:dpfedavg}
    \theta_{t+1}\gets\theta_t+\frac{1}{N}\left( \sum\nolimits_{i\in[N]} \Delta_i / \max(\frac{\|\Delta_i\|_2}{C},1) +\mathcal{N}(\mathbf{0},z^2C^2\mathbf{I}) \right),
\end{equation}
where $C$ is the gradient clipping threshold, and $z$ is the noise multiplier determined by the privacy accountant with given $\epsilon$, $\delta$, and training steps. The noise multiplier $z$ indicates the amount of noise required to reach a particular privacy budget. To privatize the participation of clients in FL, the noise added for client-level DP typically correlates with the number of clients. This incurs a large magnitude of noise in cross-silo FL in the medical field, which can significantly deteriorate the final server model performance.

\subsection{Adaptive Intermediary for Improving Client-Level DP}

The key to optimizing the privacy-performance trade-off lies in mitigating the effects of noise without compromising privacy protection. Based on the noise calculation in Eq.~(\ref{eq:dpfedavg}), we propose to study the final effects of noise on the server model, which can be denoted as $\zeta \sim  \mathcal{N}(0,\sigma^2\mathbf{I})$, where $\sigma = \sfrac{zC}{N}$. Note that the final noise ($\zeta$) is determined by $\sigma$, which relates to the noise multiplier $z$, clip threshold $C$, and the number of clients $N$. In DP, the clip threshold and the noise multiplier are usually pre-assigned. Therefore, the noise level can be reduced by increasing the number of clients $N$. To this end, we propose to reduce the noise by splitting the original clients into non-overlapping sub-clients, which serve as intermediaries to communicate with the server (see Fig.~\ref{fig:method}~(a)). We validate our hypothesis by studying the feasibility and analyzing the relationships between the intermediary number, noise, and performance.

\subsubsection{Feasibility.} We demonstrate the feasibility by showing the use of intermediary preserves privacy. For $\mathcal{X}$ the collection of possible datasets from extant clients, denote $D_i\in\mathcal{X}$ the dataset of client $i$, we randomly split $D_i$ into $v$ disjoint subsets $D_{i,1}, ..., D_{i,v}$, so that $\sqcup_j D_{i,j}=D_i$. We define the dataset $D_{i,j}$ of client $i$ as the intermediary $j$. Then we show that partitioning extant clients into multiple intermediaries is capable of maintaining DP. Denote the collection of all possible datasets formed by the intermediaries as $\mathcal{Y}$, and note that $\mathcal{X}\subseteq\mathcal{Y}$. We have:
\begin{theorem}

\label{thm:dp_split}
If a randomized learning mechanism $\mathcal{M}:\mathcal{X}\rightarrow \mathcal{R}$ is $(\epsilon,\delta)-DP$, then its induced mechanism $\Tilde{\mathcal{M}}:\mathcal{Y}\rightarrow \mathcal{R}$ is also $(\epsilon,\delta)$-DP. 

\end{theorem}
This indicates that partitioning the original client into intermediaries keeps the same DP regime. The proof can be found in Appendix~\ref{app:secB}. We also analyze the reverse relation in the appendix section to complete the overall relationship.

\subsubsection{Privacy-performance trade-off analysis.}
With the above basis, we further investigate the privacy-performance trade-off by varying the number of intermediaries. According to the noise calculation of $\sigma = \sfrac{zC}{N}$, we can reduce noise by splitting clients into intermediaries to increase $N$. However, increasing the number of intermediaries causes each intermediary to hold fewer samples. This may affect the aggregation direction and harms final performance consequently. There is a trade-off behind intermediary splitting.
To investigate the trade-off, we design and study two highly related metrics, i.e., noise level $\xi$ and client update diversity level $\varphi$. Denoting clipped gradients as $\hat{\Delta}_i$, we define the noise level and diversity level as:
\begin{equation}
    \xi = \frac{\|\zeta\|}{\|\sum\nolimits_{i\in N}\hat{\Delta_i}\|_2}, \quad \varphi = \frac{\sum_{i \in N}\left\|\Delta_{i}\right\|_2}{\|\sum_{i \in N} \hat{\Delta_{i}}\|_2}.
\end{equation}
By varying the number of intermediaries, we obtain different values for noise levels and diversities (see Fig.~\ref{fig:method}~(b)). By fitting the relations between noise level (client update diversity) and the number of intermediaries for each client (denoted as $v$), we surprisingly find the relations that:
\begin{equation}
    \xi_v = v^{-1} \cdot \xi, \quad \varphi_v = v \cdot \varphi,
\end{equation}
where $\xi_v$ and $\varphi_v$ denote the value when each client is split into $v$ intermediaries. By defining the intermediary ratio as $\lambda = \sfrac{\xi}{\varphi}$, we can use this ratio to quantify the relations between noise level and diversity, which helps identify the optimal number of intermediaries to generate.
\\\\
\textbf{Adaptive intermediary generation.}
We can generate the intermediary based on the defined intermediary ratio $\lambda$. 
We experimentally investigated the relationships between the final performance and the number of intermediaries and found the optimal ratio lies in the range of $\sfrac{1}{N}$. Therefore, for each client, the number of intermediaries is $v = \sqrt[2]{N\cdot\sfrac{\xi}{\varphi}}$. 
Considering the extreme case of $\lim_{N\rightarrow\infty} \xi = 0$, we can also infer the ratio $\lambda=0$, which further validates the rationality and consistency with our empirical findings.
For the practical application, we can initialize the number of intermediaries via the first round results. Then, for each round, we will re-calculate the ratio using $\xi$ and $\varphi$ from the last round, and then adaptively split clients to make sure the new ratio lies around $\sfrac{1}{N}$.

\subsection{Cumulation of Sample-Level DP to Client-Level}

We further investigate the relationships between client-level DP and sample-level DP, by cumulating sample-level DP mechanism to a client level. In DP-SGD~\cite{Abadi2016DLDP}, denote the standard deviation of Gaussian noise as $\sigma = z(\epsilon,\delta)c/K$ with $K$ being the batch size, $c$ being the sample-level gradient clip bound and $z$ being the noise multiplier determined by privacy accountant with $(\epsilon, \delta)$. Noise is added to each batch gradient before taking a descent, so that each step is $(\epsilon, \delta)$-DP.

Note that $z$ can take different forms, the form provided by moment accountant \cite{Abadi2016DLDP} is $z(\epsilon,\delta)=\mathcal{O}(\sqrt{\ln(1/\delta)/\epsilon^2})$. Through the use of the moment accountant and sensitivity cumulation, we can calculate the standard deviation of cumulated noise in $\Tau$ steps as $\sigma_\Tau = z'(\epsilon_\Tau,\delta_\Tau)\mathcal{S}_\Tau$, where $\epsilon_\Tau=\mathcal{O}(\sqrt{\Tau}\epsilon)$, $\delta_\Tau=\mathcal{O}(\delta)$, and $\mathcal{S}_\Tau=\mathcal{O}(\Tau c)$. It follows that $z'=\mathcal{O}(1/\epsilon_\Tau) =\mathcal{O}(z/\sqrt{\Tau})$, and $\sigma_s'=\mathcal{O}(\sqrt{\Tau}\sigma_s)$. This indicates that the noise scale cumulates at a rate of $\mathcal{O}(\sqrt{\Tau})$. With regards to performance, we prove in Appendix~\ref{app:sec:c} that the variance of the difference between the noisy model and the original model diverges with a rate of $\mathcal{O}((1-2\eta\beta+\eta^2\mu^2)^{\Tau} )$ for $\mu$-convex, $\beta$-smooth loss functions. This shows that increasing $\Tau$ also increases the probability of obtaining a model which diverges further from the original model, resulting in poorer performance. 
\\
\\
\textbf{On the client-level. } For client-level noise, we can compute the standard deviation as $\sigma_c=z(\epsilon_c,\delta_c)C$, where $C$ is the clip bound of client update. The clip bound is typically set to the median among $l2$-norms of all client updates. Assuming an identical distribution across clients and samples, we have $C=\mathcal{O}(\Tau c)$. As a result, we have $z_c=\mathcal{O}(z_\Tau)$, indicating that the cumulation of sample-level noise in DP-SGD gives the same DP level up to a constant, which is equivalent to adding noise directly to the client level through the moment accountant. Regarding the performance, we note that by leveraging the noisy models from several clients that hold identically distributed datasets, we can reduce the probability of getting a significantly drifted model without additional privacy leakage.

\section{Experiment}

\subsection{Experimental Setup}

\textbf{Datasets.}
We evaluate our method on two tasks: 1) intracranial hemorrhage (ICH) classification, and 2) prostate MRI segmentation. For ICH classification, we use the RSNA-ICH dataset~\cite{ich} and follow \cite{kyung2022improved} to relieve the class imbalance across ICH subtypes and perform the binary diseased-or-healthy classification. We randomly sample 25,000 slices and split them into 20 clients, where each client data is split into 60\%, 20\%, and 20\% for training, validation, and testing. We resize images to $224\times224$ and perform data augmentation with random affine and horizontal flip. For prostate segmentation, we adopt a multi-site T2-weighted MRI dataset~\cite{liu2020ms} which contains 6 different data sources from 3 public datasets \cite{lemaitre2015computer,litjens2014evaluation,nicholas2015nci}. We regard each data source as one client, resize images to $256\times 256$, and use 50\%, 25\%, and 25\% for for training, validation and testing.
\\\\\textbf{Privacy setup.} We use the Opacus'~\cite{yousefpour2021opacus} implementation of privacy loss random variables (PRVs) accountant~\cite{gopi2021numerical} for the Gaussian mechanism for our privacy accounting. We restrict the total number of training rounds and then account for any privacy overheads with various privacy levels controlled by the noise multiplier $z$, where a higher $z$ indicates a higher privacy regime $\epsilon$. Adaptive clipping \cite{andrew2021differentially} is employed to bound each client's contribution in the federation. Following~\cite{zheng2021federated}, we report the results by exploring effects of different noise multiplier $z$ values. We set $z$ in the range of $\{0.5,1.0,1.5\}$ for ICH diagnosis, and $\{0.3,0.5,0.7\}$ for prostate segmentation, which induces privacy budgets of $\{245.6,72.4,36.9\}$ and $\{597.3,224.7,119.4\}$, respectively. We set $\delta=10^{-k}$ where $k$ is the smallest integer that satisfies $10^{-k}\leq1/n$ for the client number $n$ as suggested by~\cite{li2023differentially}. 
\\\\\textbf{Implementation details.} We use Adam optimize, set the local update epoch to 1, and set total communication rounds to 100. We use DenseNet121~\cite{huang2017densely} for classification, the batch size is 16 and the learning rate is $3\times 10^{-4}$. We use UNet~\cite{ronneberger2015u} for segmentation, the batch size of 8, and the learning rate is $10^{-3}$. 

\begin{table}[!tp]
\newcommand{\highlight}{\rowcolor[gray]{0.93}}
\centering
\caption{Performance comparison of different DP optimization methods and ours. We report mean and standard deviation across three independent runs with different seeds.}
\label{tab:results}
\setlength{\tabcolsep}{1pt}
\setlength{\aboverulesep}{0pt}
\setlength{\belowrulesep}{0pt}
\resizebox{1\textwidth}{!}{%
\noindent\begin{tabular}{l c cc c cc c cc c cc}
\toprule

\multicolumn{13}{c}{Intracranial Hemorrhage Diagnosis ($N=20$)}                                                                \\ \hline
\multirow{2}{*}{Method} && \multicolumn{2}{c}{No Privacy} && \multicolumn{2}{c}{$z=0.5$}  && \multicolumn{2}{c}{$z=1.0$} && \multicolumn{2}{c}{$z=1.5$} \\
 \cmidrule{3-4} \cmidrule{6-7} \cmidrule{9-10} \cmidrule{12-13} 
 && AUC $\uparrow$ & {Acc} $\uparrow$ && {AUC} $\uparrow$  & {Acc} $\uparrow$ && {AUC} $\uparrow$ & {Acc} $\uparrow$  && {AUC} $\uparrow$ & {Acc} $\uparrow$ \\ \hline
DP-FedAvg~\cite{brendan2018learning} && $90.88_{\pm 0.15}$     & $82.85_{\pm 0.26}$      &&           $70.38_{\pm 0.61}$                  &          $64.94_{\pm 0.28}$                 & &        $68.00_{\pm 1.43}$                      &      $63.55_{\pm 1.12}$                     &&       $66.77_{\pm 0.12}$               &           $62.01_{\pm 0.58}$                 \\
\highlight +Ours &&      -           &  - &&  $82.42_{\pm 0.29}$ & $74.87_{\pm 0.43}$ && $80.84_{\pm 0.78}$ & $73.37_{\pm 0.92}$ && $80.77_{\pm 0.80}$ & $72.95_{\pm 0.47}$\\
\hline
DP-FedAdam~\cite{fedadam} && $91.85_{\pm0.30}$& $84.16_{\pm0.56}$&&$75.91_{\pm0.28}$&$68.75_{\pm0.16}$&&$70.75_{\pm2.18}$& $65.34_{\pm0.64}$&& $70.89_{\pm2.05}$&$63.73_{\pm1.16}$\\

\highlight +Ours &&      -           &  - &&  $82.86_{\pm 0.47}$ & $75.20_{\pm 0.27}$ && $81.63_{\pm 0.38}$ & $73.99_{\pm 0.79}$ && $80.55_{\pm 0.60}$ & $73.06_{\pm 0.68}$\\
\hline

DP-FedNova~\cite{fednova} && $90.89_{\pm0.25}$ & $83.00_{\pm0.17}$ && $71.84_{\pm1.51}$ & $66.25_{\pm0.91}$& & $69.26_{\pm1.76}$ & $63.75_{\pm1.49}$ && $68.45_{\pm0.93}$ & $63.21_{\pm1.13}$\\

\highlight +Ours &&      -           &  - && $82.73_{\pm 0.38}$ & $75.35_{\pm 0.27}$ && $80.64_{\pm 0.57}$ & $73.55_{\pm 0.78}$ && $79.39_{\pm 0.41}$ & $71.70_{\pm 0.35}$\\
\hline

$\mathrm{DP}^2$-RMSProp~\cite{li2023differentially}                  &  &         $88.89_{\pm 0.02}$                &          $80.77_{\pm 0.25}$                &&             $70.59_{\pm 1.16}$            &              $64.92_{\pm 1.19}$            &&  $67.43_{\pm 0.60}$    &   $62.05_{\pm 0.23}$               &&  $65.91_{\pm 1.40}$ & $61.77_{\pm 1.13}$          \\

\highlight +Ours &&      -           &  - && $81.60_{\pm 0.68}$ & $74.47_{\pm 0.95}$ && $80.23_{\pm 0.22}$ & $73.15_{\pm 0.44}$ && $81.32_{\pm 0.50}$ & $74.21_{\pm 0.85}$\\

\bottomrule\toprule

\multicolumn{13}{c}{Prostate MRI Segmentation ($N=6$)}                    \\ \hline
\multirow{2}{*}{Method} && \multicolumn{2}{c}{No Privacy} && \multicolumn{2}{c}{$z=0.3$}  && \multicolumn{2}{c}{$z=0.5$} && \multicolumn{2}{c}{$z=0.7$} \\
 \cmidrule{3-4} \cmidrule{6-7} \cmidrule{9-10} \cmidrule{12-13} 
 && Dice $\uparrow$ & IoU $\uparrow$ && Dice $\uparrow$ & {IoU} $\uparrow$ && {Dice} $\uparrow$  & {IoU} $\uparrow$ && {Dice} $\uparrow$& {IoU} $\uparrow$ \\ \hline
DP-FedAvg~\cite{brendan2018learning} && $87.69_{\pm 0.12}$ &  $79.62_{\pm 0.13}$  &&           $41.43_{\pm 3.89}$                  &          $29.28_{\pm 3.40}$                  &&        $22.45_{\pm 3.15}$                      &      $13.50_{\pm 2.24}$                     &&       $13.59_{\pm 0.96}$               &           $7.41_{\pm 0.59}$                 \\

 \highlight +Ours &&      -           &  - && {$70.59_{\pm 1.55}$} & {$67.72_{\pm 0.47}$} && {$63.28_{\pm 4.69}$} & {$61.12_{\pm 0.50}$} && {$58.14_{\pm 4.71}$} & {$56.18_{\pm 7.21}$}\\
\hline

DP-FedAdam~\cite{fedadam}   &&   $87.63_{\pm 0.16}$ &     $79.65_{\pm 0.20}$       &&        $38.24_{\pm 2.86}$     &     $38.24_{\pm 1.38}$           &&     $16.50_{\pm 1.82}$     &        $15.03_{\pm 2.28}$      &&   $9.15_{\pm 2.19}$   &    $5.49_{\pm 2.22}$      \\
 
\highlight +Ours && - & -  && {$69.68_{\pm 1.45}$} & {$61.31_{\pm 0.71}$} && {$57.11_{\pm 6.30}$} & {$57.23_{\pm 1.17}$} && {$43.99_{\pm 9.04}$}& {$47.49_{\pm 7.81}$}\\
\hline

DP-FedNova~\cite{fednova} &&   $87.44_{\pm 0.35}$ &   $79.49_{\pm 0.29}$ &&   $41.91_{\pm 6.34}$ &   $29.33_{\pm 5.78}$ &&   $17.10_{\pm 7.45}$ &   $9.96_{\pm 4.81}$ &&   $11.41_{\pm 0.64}$ &   $6.06_{\pm 0.38}$\\

\highlight +Ours && - & - && $70.80_{\pm 1.28}$ & $66.64_{\pm 1.42}$ && $68.63_{\pm 2.17}$ & $63.99_{\pm 0.96}$ && $58.99_{\pm 4.43}$ & $59.14_{\pm 2.03}$\\
\hline

$\mathrm{DP}^2$-RMSProp~\cite{li2023differentially}              && $87.46_{\pm 0.08}$ &   $80.00_{\pm 0.09}$                      &&    $38.33_{\pm 2.44}$                      &      $24.73_{\pm 3.80}$                   &&                $16.74_{\pm 0.83}$          &            $10.75_{\pm 0.99}$             &&           $7.77_{\pm 0.41}$              &    $4.00_{\pm 0.23}$                    \\
\highlight +Ours && - & - &&  $63.05_{\pm 2.60}$ & $63.05_{\pm 2.60}$ && $53.53_{\pm 4.97}$ & $60.84_{\pm 5.81}$ && $47.82_{\pm 2.01}$ & $59.66_{\pm 2.76}$\\

\bottomrule
\end{tabular}}
\end{table}

\subsection{Empirical Evaluation}

First, we present experimental results using different global optimizers on the server with client-level DP. Then, we demonstrate how our adaptive intermediary strategy benefits privacy-performance trade-offs. We consider four popular private server optimizers: {DP-FedAvg}~\cite{brendan2018learning} which adds client-level privacy protection to FedAvg~\cite{fedavg}, {DP-FedAdam} which is a differentially private version of the optimizer FedAdam~\cite{fedadam}, {DP-FedNova} which we equip the global solver FedNova~\cite{fednova} for client-level DP, and {DP$\mathbf{^2}$-RMSProp}~\cite{li2023differentially} which is a very recent private optimization framework and we deploy it as the global optimizer in FL.

We perform validation with different noise multiplier values. Non-private FL is also provided as a performance upper bound. Note that our method has the same performance ascompared methods in non-private settings, because there are no noises to harmonize. As can be observed from Table~\ref{tab:results}, severe performance degradation occurs in the private cross-silo FL setting, especially for high-privacy regimes (e.g., $z=0.7$ for prostate segmentation). There are no significant differences among different global optimizers, which shows that the optimizers carefully designed for non-private FL are unable to address the noisy gradient issue in DP settings. However, our method relieves the gradient corruption and consistently and substantially boosts performance even with large noises (e.g., 44.55\% Dice boost on prostate segmentation with $z=0.7$). We also identify that the influences on performance introduced by DP may vary across different tasks and client numbers. For example, the segmentation task with fewer clients is more seriously damaged compared with the classification task with more clients.

\begin{figure*}[t]
\centering
\includegraphics[width=0.99\columnwidth]{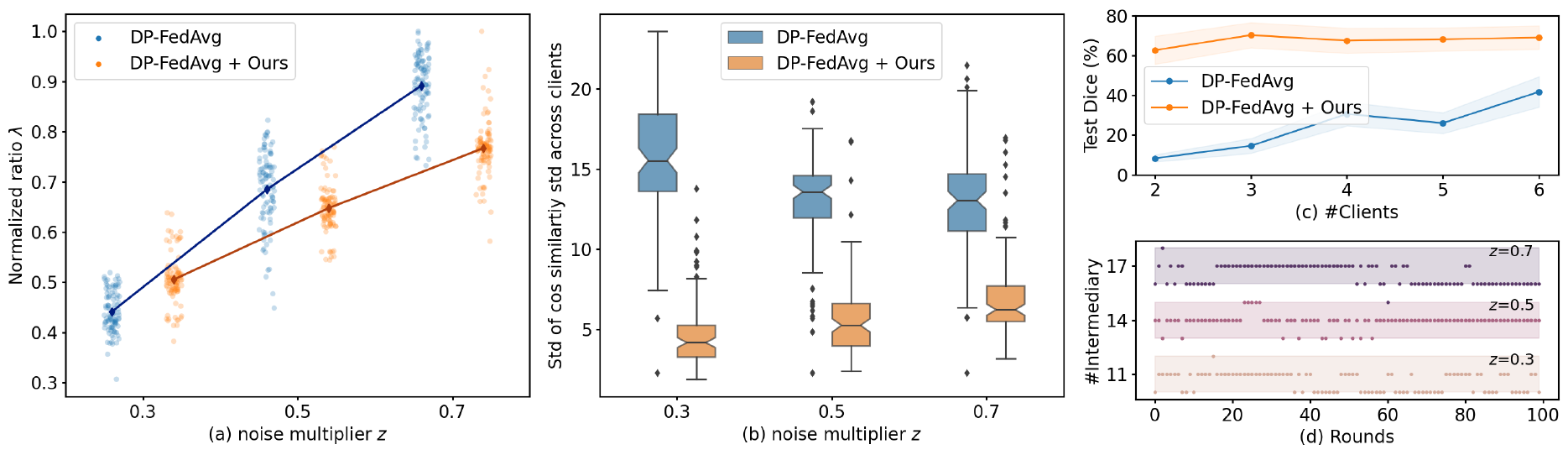}
\caption{Analytical studies on prostate segmentation. (a) The distribution of normalized ratio $\lambda$ across communication rounds under different privacy levels. (b) The std of cosine similarities between $\Delta_i$ and the aggregated gradients in each round under different privacy levels. (c) Performance with different client numbers. (d) Intermediary variations across training rounds under different privacy regimes.}
\label{fig:analytical_study}
\end{figure*}

\subsection{Analytical Studies}
\label{sec:analysis}
\textbf{Effects of optimizing privacy-performance trade-offs.} We present the dynamic behavior of our method regarding variations of the intermediary ratio $\lambda$ across different rounds in Fig.~\ref{fig:analytical_study} (a). Compared with DP-FedAvg~\cite{brendan2018learning}, where $\lambda$ shows a significant increase with the rise of noise multiplier $z$, our method harmonizes this trend with more centralized distributions by the adaptive intermediary for better privacy-performance trade-offs. In Fig.~\ref{fig:analytical_study} (b), we also study the standard deviation of similarities, which is another metric for quantifying gradient diversity between local and global gradients. Our method shows more stable optimization directions with less variance among clients. Moreover, we observe a decline in gradient diversities as the privacy regime rises for DP-FedAvg \cite{brendan2018learning}. To interpret, we speculate that local optimization may be dominated by greater noises for more common gradient de-corruption.
\\\\\textbf{Client scalability analysis.} As the noise level is highly dependent on client numbers (see Eq.~(\ref{eq:dpfedavg}) and Table~\ref{tab:results}), we investigate the scalability of DP-FedAvg~\cite{brendan2018learning} and our method by varying number of clients. Fig.~\ref{fig:analytical_study} (c) presents the results on prostate segmentation with different training clients ($z=0.3$). Notably, we keep test data unchanged for fair comparisons. We observe a dramatic drop in performance of DP-FedAvg~\cite{brendan2018learning} due to excessive noise when the number of clients shrinks. However, our method performs stably even under extreme conditions, e.g., the federation only has two participants.
\\\\\textbf{Stability of Adaptive Intermediary Estimation.} Finally, we analyze the historical variation of our adaptive intermediary strategy in Fig.~\ref{fig:analytical_study} (d), where we present the intermediary numbers during the training progress. We expect that more intermediaries are required to balance the privacy-performance trade-off with a greater noise multiplier $z$. Besides, we verify the reliability and stability of our adaptive intermediary estimation by showing that the variation during the training does not exceed one, except for a single instance when $z=0.7$.

\section{Conclusion}
In this paper, we propose a novel adaptive intermediary method to promote privacy-performance trade-offs in the context of client-level DP in FL. We have comprehensively studied the relations among number of intermediaries, noise levels and training diversities in our work. We also investigate relations between sample-level and client-level DP. Our proposed method outperforms compared methods on both medical image diagnosis and segmentation tasks and shows good compatibility with existing DP optimizers.
For future work, it is promising to investigate our method for clients with imbalanced class distributions, where the intermediary may not have all labels.
\\
\\
\textbf{Acknowledgement.} This work was supported in part by Shenzhen Portion of Shenzhen-Hong Kong Science and Technology Innovation Cooperation Zone under HZQB-KCZYB-20200089, in part by National Natural Science Foundation of China (Project No. 62201485), in part by National Key R\&D Program of China Project 2022ZD0161100, in part by Hong Kong Innovation and Technology Commission Project No. ITS/238/21, in part by Science, Technology and Innovation Commission of Shenzhen Municipality Project No. SGDX20220530111201008, in part by Hong Kong Research Grants Council Project No. T45-401/22-N, and in part by NSERC Discovery Grant (DGECR-2022-00430).

\bibliographystyle{splncs04}
\bibliography{references}

\clearpage
\appendix

\newcommand{\M}{\mathcal{M}} 
\setcounter{table}{0}
\setcounter{figure}{0}
\setcounter{theorem}{0}
\setcounter{definition}{0}
\renewcommand\thetable{\thesection.\arabic{table}}
\renewcommand\thefigure{\thesection.\arabic{figure}} 
\renewcommand\thetheorem{\thesection.\arabic{theorem}}    
\renewcommand\thedefinition{\thesection.\arabic{definition}}    


\section{Notation Table}
\begin{table}[h]
\centering
\caption{Notations occurred in the paper.}
\begin{tabular}{cl}
\hline
Notations & Description \\ \hline
$\epsilon$ & privacy budget, $\epsilon \in \mathbb{R}$ \\
$\delta$ & probability of $\epsilon$-DP failure, $\delta \in [0,1]$ \\
$z$ & the privacy accountant, $z=z(\epsilon,\delta)\in\mathbb{R}$\\
$C$ & clip bound of the client gradient, which typically takes the value of the\\
    & medium or maximum of all the gradients \\
$c$ & clip bound of the batch gradient in DP-SGD\\
$\mathcal{S}(\mathcal{M})$ & sensitivity of the mechanism $\mathcal{M}$\\
$\sigma$ & standard deviation of the Gaussian noise added, $\sigma=z\mathcal{S} \in \mathbb{R}$ \\
$\zeta$ & the Gaussian noise added, $\zeta \in \mathbb{R}$ \\
$\mathcal{X},\mathcal{Y}$ & collection of all possible datasets from clients, sub-clients\\
$N; N_v$ & total number of clients, sub-clients, $N, N_v\in \mathbb{N} $ \\
$v$ & number of sub-clients that each client is split into, $v\in\mathbb{N}$ \\
$\Delta; \Delta'; \Tilde{\Delta'}$ & client gradient; clipped client gradient; clipped client gradient with noise \\
$t, T$ & number, total number of communication round in FL, $t, T\in\mathbb{N}$ \\
$\tau, \Tau$ & number, total number of descent made in SGD, $\tau, \Tau\in\mathbb{N}$ \\
$\{x_1, ..., x_K\}$ & data sample\\
$\theta$ & model parameter, $\theta \in \mathbb{R}^d$ \\

\hline
\end{tabular}%
\label{tab:notation}
\end{table}

\section{On client and sub-client level privacy}
\label{app:secB}

This section is to investigate the relationship between client level and sub-client level privacy. We first give the precise definitions of the relevant terms.

\begin{definition}{(Client-level Adjacent Dataset)}
    For $\mathcal{X}$ the collection of all feasible combination of clients who participated in model training, two datasets $D,D'\in\mathcal{X}$ are adjacent (denoted by $D\sim_c D'$, or $D\sim D'$ if the context is clear) if they differ by one participating client by replacement.  
\end{definition}

\begin{definition}{($(\epsilon, \delta)$-Differential Privacy)}
For a randomized learning mechanism $\M$: $\mathcal{X} \rightarrow \mathcal{R}$, where $\mathcal{X}$ is the collection of datasets it can be trained on, and $\mathcal{Y}$ is the collection of model it can generate, it is $(\epsilon, \delta)$-DP if:
$$(\forall S \subseteq \mathcal{R})(\forall d,d'\in \mathcal{X}, d\sim d') 
\operatorname{Pr}[\M(d) \in S] \leq \exp 
(\varepsilon) \cdot 
\operatorname{Pr}\left[\M\left(d^{\prime}\right) \in 
S\right]+\delta $$
\end{definition}

\begin{definition}{(Sub-client)}
    For $\mathcal{X}$ the collection of possible datasets from extant clients, for each $D_i\in\mathcal{X}$ the dataset of client $i$, we randomly split $D_i$ into $n$ disjoint subsets $d_{i,1}, ..., d_{i,n}$, so that $\sqcup_j d_{i,j}=D_i$. We consider the imaginary client holding the dataset $d_{i,j}$ as a sub-client of client $i$. 
\end{definition}

The collection of all possible datasets formed by the sub-clients is denoted by $\mathcal{Y}$. Note that $\mathcal{X}\subseteq\mathcal{Y}$.

\begin{theorem}
\label{app:thm:dp_split}
If a randomized learning mechanism $\mathcal{M}:\mathcal{X}\rightarrow \mathcal{R}$ is $(\epsilon,\delta)-DP$, then its induced mechanism $\Tilde{\mathcal{M}}:\mathcal{Y}\rightarrow \mathcal{R}$ is also $(\epsilon,\delta)$-DP. 
\end{theorem}

\begin{proof}
    Consider arbitrary $d,d'\in\mathcal{Y}, d\sim d'$, they differ by one sub-client, call them $i$ and $i'$. Now for a mechanism $\tilde{\mathcal{M}}$ trained on them, it is induced from a mechanism trained on dataset in $\mathcal{X}$. So $d$ correspond to a dataset $D\in\mathcal{X}$, and the sub-client $i$ belongs to some client $I$. \\
    
    \noindent Now for $D$, to get $D'$, we replace client $I$ with a surrogate client $I'$ by replacing the sub-client $i$ with $i'$. In this way, we have $D'=d'$ and $D\sim D'$. It follows that: 
    \begin{align*}
        \operatorname{Pr}[\tilde{\M}(d) \in S] &= \operatorname{Pr}[\M(D) \in S]\\
        &\leq \exp (\epsilon) \cdot \operatorname{Pr}\left[\M\left(D^{\prime}\right) \in S\right]+\delta\\
        &=  \exp (\epsilon) \cdot \operatorname{Pr}\left[\tilde{\M}\left(d^{\prime}\right) \in S\right]+\delta
    \end{align*}
    showing that the mechanism $\tilde{\mathcal{M}}$ is $(\epsilon,\delta)$-DP on $\mathcal{Y}$. \qed
\end{proof}

For the reverse relation, the most direct approach is to use group privacy. 

\begin{theorem}{(Group Privacy)} 
    If $\mathcal{M}:\mathcal{Y}\to\mathbb{R}$ is $(\epsilon,\delta)$-DP, and that each client is divided into $n$ sub-clients, then $\mathcal{M|_\mathcal{X}}:\mathcal{X}\to\mathbb{R}$ is $(n\epsilon,\frac{1-\epsilon^n}{1-\epsilon}\delta)$-DP.
\end{theorem}
\begin{proof}
    For $D,D'\in\mathcal{X}, D\sim D'$, there exists $d,d'\in\mathcal{Y}$ s.t. $d$ and $d'$ differ by $n$ sub-clients. Then we can constuct a sequence of $\{d_i\}_{i=1}^n$ s.t. $d_1=d, d_n=d'$ and $d_i\sim d_{i+1},\forall 1\leq i\leq n-1$ by varying one sub-client each time. \\
    
    \noindent Then we have 
    \begin{align*}
        \operatorname{Pr}[\M(D) \in S] &= \operatorname{Pr}[\M(d_1) \in S]\\
        &\leq \exp (\epsilon) \cdot \operatorname{Pr}\left[\M\left(d_2\right) \in S\right]+\delta \\
        &\leq \exp (\epsilon) \cdot \Big(\exp (\epsilon) \cdot \operatorname{Pr}\left[\M\left(d_2\right) \in S\right]+\delta\Big)+\delta\\
        &\leq ... \text{ (substituting the inequalities in recursively)}\\
        &\leq \exp (n\epsilon) \cdot \operatorname{Pr}\left[\M\left(d_n\right) \in S\right] + \sum_{i=0}^{n-1}\epsilon^i\delta\\
        &= \exp (n\epsilon) \cdot \operatorname{Pr}\left[\M\left(D'\right) \in S\right] + \frac{1-\epsilon^n}{1-\epsilon}\delta
    \end{align*} \qed
\end{proof}

Now with a further assumption that the algorithm can access at most one sub-client from each client in each round of training, we have the following reverse relation.

\begin{theorem}{(Composition theorem, Kairouz et al. 2015)}
    If a randomized learning mechanism $\mathcal{M}:\mathcal{Y}\rightarrow \mathcal{R}$ satisfying the above assumption is $(\epsilon,\delta)-DP$, then $\mathcal{M}|_\mathcal{X}:\mathcal{X}\rightarrow \mathcal{R}$ is $(\epsilon'(d),1-(1-\delta)^k(1-d))$-DP, with $d\in[0,1]$ and
    
    $$\epsilon'=\min\left\{k\epsilon, \frac{(e^\epsilon-1)\epsilon k}{e^\epsilon+1}+\epsilon\sqrt{2k\ln\Bigl(e+\frac{\sqrt{k\epsilon^2}}{d}\Bigr)}, \frac{(e^\epsilon-1)\epsilon k}{e^\epsilon+1}+\epsilon\sqrt{2k\ln\Bigl(\frac{1}{d}\Bigr)}{} \right\}$$
\end{theorem}

For the proof, refer to Theorem 3.4 in \cite{kairouz2015compo}.

\section{Relation between sample-level and client-level DP}
\label{app:sec:c}

\subsubsection{Sample-level DP algorithm}

We first give the details of the sample-level DP-SGD algorithm we are considering. 

\begin{algorithm}[!h]
\caption{DP-SGD algorithm in one round of training}\label{alg:dp_sgd}
\textbf{Input: } Learning rate $\eta_t$, noise scale $\sigma$, gradient norm bound $C$, batch size $K$, data sample$\{x_1,...,x_N\}$, initial model $\theta_0$
\begin{algorithmic}
\State $T \gets N/K$
    \For{$t\in[T]$}
    \State Collect a random batch $X$ of data by sampling with probability $K/N$
        \For{$i \in X$}
            \State $\textbf{g}_t(x_i) \gets \nabla_{\theta_t}L(\theta_t,x_i)$
            \State $\Bar{\textbf{g}}_t(x_i) \gets \textbf{g}_t(x_i)/\max\{1, \frac{\Vert \textbf{g}_t(x_i)\Vert_2}{c}\}$
        \EndFor
        \State $\tilde{\textbf{g}}_t \gets \frac{1}{K}\Big(\sum_i\bar{\textbf{g}}_t(x_i)+\mathcal{N}(0,z^2c^2\textbf{I})\Big)$
    \State $\theta_{t+1} \gets \theta_t - \eta \tilde{\textbf{g}}_t$
    \EndFor
\end{algorithmic}
\textbf{Output:} $\theta_T$
\end{algorithm}

Note that it is feasible to use a different value of $\eta_t$ for each descent. In this section, we will consider the algorithm with a uniform step size $\eta$, while the theorems naturally generalise to variable step sizes. \\

We first quote the theorem on privacy analysis of DP-SGD by Abadi et al.\cite{Abadi2016DLDP}. 

\begin{theorem}{(Moment accountant)}
    If $\sigma$ is chosen such that each step is $(\epsilon,\delta)$-DP to the batch, then the DP-SGD algorithm is $(\mathcal{O}(\epsilon\sqrt{\Tau},\delta)$-DP, where $\Tau$ is the number of descents taken.
\end{theorem}
For more details, refer to Theorem 1 in \cite{Abadi2016DLDP}.Note that the sampling rate is taken to be $1$ in this setting since instead of randomly sampling 'lots' as in the original paper, we use pre-determined bateches. \\

We then investigate the cumulation effect of sample-level Gaussian mechanism. 

\subsubsection{Definition and Lemmas}\

We first give a few definitions and lemmas on some useful properties.\\ 

\noindent\textbf{(A1)} A function $f(\bm{x})$ is $\mu$\textbf{-convex} if $\mu\geq 0$ and $f$ satisfies: 
$$\langle \nabla f(\bm{x}),\bm{y}-\bm{x}\rangle \leq -(f(\bm{x}) - f(\bm{y}) + \frac{\mu}{2} \Vert \bm{x}-\bm{y}\Vert^2), \text{ for any } \bm{x},\bm{y}.$$
And if $f$ is twice-differentiable, then this means $\Vert\nabla^2 f(\bm{x})\Vert\geq \mu, \forall \bm{x}$.\\
We say $f$ is $\mu$\textbf{-strongly convex} if $\mu>0$.\\

\noindent\textbf{(A2)} A function $f(\bm{x})$ is $\beta$\textbf{-smooth} if it satisfies: 
$$\Vert \nabla f(\bm{x})-\nabla f(\bm{y})\Vert \leq \beta \Vert \bm{x}-\bm{y}\Vert, \text{ for any } \bm{x},\bm{y}.$$
Note that this implies:
$$f(\bm{y})\leq f(\bm{x}) +\langle \nabla f(\bm{x}),\bm{y}-\bm{x}\rangle + \frac{\beta}{2} \Vert \bm{x}-\bm{y}\Vert^2$$
And if $f$ is twice-differentiable, then this means $\Vert\nabla^2 f(\bm{x})\Vert\leq \beta, \forall \bm{x}$.\\

\begin{lemma}
    For $x_t = ax_{t-1}+b,a\neq 1$, $x_0=0$, we have $x_t = \frac{a^t-1}{a-1}b$.
\end{lemma}
\begin{proof}
    We prove by induction. Firstly, the theorem holds trivially for $t=0$ or $1$. \\

    \noindent Now assume the statement holds for all $t\leq n$, then for $t=n+1$, we have:
    \begin{align*}
        x_{t}=ax_n+b &=a\times \frac{a^n-1}{a-1}b+b\\
        &= \Big(\frac{a^{n+1}-a}{a-1}+1\Big)b\\
        &= \frac{a^{t}-1}{a-1}b
    \end{align*}\qed
\end{proof}

\begin{lemma}(some technical lemmas)~
\begin{enumerate}
    \item $\E(AB)=\E(A)\E(B)\text{ for $A$, $B$ independent }$.\\
    \item (Cauchy-Schwarz inequality) $|\langle \bm{x}, \bm{y} \rangle| \leq \Vert \bm{x}\Vert \Vert \bm{y} \Vert$. \\
    \item (Generalized intermediate value theorem) For $f:X\to Y$ a continuous map and $E\subset X$ a connected subset, we have $f(E)$ is connected.\\
    \item If a function $f$ is $\beta$-smooth, then it is also continuous.
\end{enumerate}
\end{lemma}

These are standard results, and the proofs are omitted. \\

\subsubsection{Cumulation of sample-level Gaussian Mechanism}

\begin{theorem}{(Cumulative Sensitivity)}
    For the DP-SGD algorithm as described above, we have $\mathcal{S}(\theta_t)\leq 2\eta tc$. 
\end{theorem}
\begin{proof}
Note that $\mathcal{S}(\theta_0)=0$ and
\begin{align*}
    \mathcal{S}(\theta_t) &= \mathcal{S}(\theta_{t-1}-\eta \tilde{\textbf{g}}_{t-1})\\
    &= \max_{d\sim d'}\{(\theta_{t-1}-\eta \tilde{\textbf{g}}_{t-1})|_d-(\theta_{t-1}-\eta \tilde{\textbf{g}}_{t-1})|_{d'}\}\\
    &\leq \max_{d\sim d'}\{\theta_{t-1}|_d-\theta_{t-1}|_{d'}\}+\max_{d\sim d'}\{\eta\tilde{\textbf{g}}_{t-1}|_d-\eta\tilde{\textbf{g}}_{t-1}|_{d'}\}\\
    &= \mathcal{S}(\theta_{t-1})+\eta \mathcal{S}(\tilde{\textbf{g}}_{t-1})
\end{align*}
while
\begin{align*}
    \mathcal{S}(\tilde{\textbf{g}}_{t-1}) &=\mathcal{S}\Big(\frac{1}{K}\Big(\sum_{i\in[K]}\bar{\textbf{g}}_{t-1}(x_i)+N_t\Big)\Big)
     =\mathcal{S}\Big(\frac{1}{K}\Big(\sum_{i\in[K]}\bar{\textbf{g}}_{t-1}(x_i)\Big)\Big)\\
    & \leq \frac{1}{K} (\sum_{i\in[K]}\mathcal{S}(\bar{\textbf{g}}_{t-1}(x_i))\Big) =2c
\end{align*}
It follows that $\mathcal{S}(\theta_t)\leq 2\eta tc$. \qed 

\end{proof}

\begin{theorem}{(Cumulative Variance)}
    For $\mu$-convex and $\beta$-smooth loss function $L$ with $\Vert\nabla L\Vert\leq c$, the variance between the noisy model and the original model diverges with rate $\mathcal{O}\Big((1-2\eta\beta+\eta^2\mu^2)^{\Tau} \Big)$, where $\Tau$ is the number of descents taken. 
\end{theorem}

\begin{proof}

Let $\tilde{\theta}_{t}$ denote the parameters of the noisy model, and $\theta_t$ denote the parameters of the original model. \\

\noindent Consider
    $$\tilde{\theta}_{t+1} = \tilde{\theta}_t - \frac{\eta}{K}\Big(\sum_{i\in[K]}\nabla L(\Tilde{\theta}_t;x_i)+N_t \Big)$$
    $$\theta_{t+1} = \theta_t - \frac{\eta}{K}\Big(\sum_{i\in[K]}\nabla L(\theta_t;x_i) \Big)$$
    where $N_t\sim \mathcal{N}(0,\sigma^2\textbf{I})$ is the Gaussian noise added, so 
    $$\tilde{\theta}_{t+1}-\theta_{t+1} = \tilde{\theta}_t-\theta_t - \frac{\eta}{K}\sum_{i\in[K]}[\nabla L(\Tilde{\theta}_t;x_i)-\nabla L(\theta_t;x_i)]+\frac{\eta}{K}N_t $$
    Consider the variance, where the expectation is taken over all noise, given data sample $\{x_1,...,x_K\}$:
    \begin{align*}
        \E\Vert\Tilde{\theta}_{t+1}-\theta_{t+1}\Vert^2 &= \E\Vert \tilde{\theta}_t-\theta_t - \frac{\eta}{K}\sum_{i\in[K]}[\nabla L(\Tilde{\theta}_t;x_i)-\nabla L(\theta_t;x_i)]+\frac{\eta}{K}N_t \Vert^2\\
        &= \E\Vert \tilde{\theta}_t-\theta_t - \frac{\eta}{K}\sum_{i\in[K]}[\nabla L(\Tilde{\theta}_t;x_i)-\nabla L(\theta_t;x_i)]\Vert^2 + \E\Vert\frac{\eta}{K}N_t \Vert^2\\
        &= \E\Vert \tilde{\theta}_t-\theta_t\Vert^2 + \frac{\eta^2}{K^2}\E\Vert\sum_{i\in[K]}[\nabla L(\Tilde{\theta}_t;x_i)-\nabla L(\theta_t;x_i)]\Vert^2 \\ & \quad + \frac{\eta^2}{K^2}\E\Vert N_t \Vert^2+\mathcal{C}
    \end{align*}
    where the second equality follows from Lemma 2.1 and 
    \begin{align*}
        \mathcal{C} &= -\frac{2\eta}{K}\E \Big\langle \tilde{\theta}_t-\theta_t, \sum_{i\in[K]}[\nabla L(\Tilde{\theta}_t;x_i)-\nabla L(\theta_t;x_i)] \Big\rangle \\
        &\geq -\frac{2\eta}{K} \E \Big(\Vert \tilde{\theta}_t-\theta_t\Vert\Vert \sum_{i\in[K]}[\nabla L(\Tilde{\theta}_t;x_i)-\nabla L(\theta_t;x_i)]\Vert \Big)\\       
        &\geq -2\eta\beta\ \E\Vert \tilde{\theta}_t-\theta_t\Vert^2
    \end{align*}
    where the second inequality follows from Lemma 2.2 and the third inequality is due to $\beta$-smoothness, Lemma 2.4, and by Lemma 2.3, there exists $x_c $ s.t. 
    \begin{align}
        \Vert\sum_{i\in[K]}[\nabla L(\Tilde{\theta}_t;x_i) - \nabla L(\theta_t;x_i)]\Vert &= K\Vert\nabla L(\Tilde{\theta}_t;x_c) - \nabla L(\theta_t;x_c)\Vert\\
        &\leq K\beta \Vert \tilde{\theta}_t-\theta_t\Vert 
    \end{align}
    where the inequality follows from the $\beta$-smoothness of $L$.\\
    
    \noindent Therefore 
    \begin{align*}
        \E\Vert\Tilde{\theta}_{t+1}-\theta_{t+1}\Vert^2 &\geq 
        \E\Vert \tilde{\theta}_t-\theta_t\Vert^2 + \eta^2\mu^2\E\Vert\Tilde{\theta}_{t}-\theta_{t}\Vert^2 + \frac{\eta^2\sigma^2}{K^2}-2\eta\beta\ \E\Vert \tilde{\theta}_t-\theta_t\Vert^2\\
        &= (1-2\eta\beta+\eta^2\mu^2)\E\Vert \tilde{\theta}_t-\theta_t\Vert^2 + \frac{\eta^2\sigma^2}{K^2} \\
        &= \frac{[(1-2\eta\beta+\eta^2\mu^2)^{t+1}-1]\eta^2\sigma^2}{(\eta^2\mu^2-2\eta\beta)K^2}
    \end{align*}
    where for the first inequality, we apply (4) again with $\mu$-convexity and substitute in the values for $\E\Vert N_t \Vert^2$ and $\mathcal{C}$, and the second equality follows from Lemma 1. 
    \qed
    
\end{proof}

\section{Additional Experiments}
\label{app:sec:d}
\subsection{Comparison with FedAvg under different training settings.}
\subsubsection{Subsampling.}
We consider the scenario where some clients may not join at every round, by further comparing our methods with FedAvg using subsampling. The subsampling means at each round, some clients may not be online, only a subset of clients are involved. We perform the experiments on the Prostate MRI segmentation and set the subsampling ratio as 2/3. The results are shown in Table~\ref{app:tab:subsample}, from the table it can be observed that the subsampling may harm the performance, while our method consistently outperforms the compared methods.

\begin{table}[!tp]
\centering
\caption{Performance comparison of FedAvg with subsampling.}
\label{app:tab:subsample}
\setlength{\tabcolsep}{5pt}
\begin{tabular}{l c c c c c c }
\toprule

\multirow{2}{*}{Method} & \multicolumn{2}{c}{$z=0.3$}  &\multicolumn{2}{c}{$z=0.5$} & \multicolumn{2}{c}{$z=0.7$} \\ 
 \cmidrule{2-7} & Dice $\uparrow$ & {IoU} $\uparrow$ & {Dice} $\uparrow$  & {IoU} $\uparrow$ & {Dice} $\uparrow$& {IoU} $\uparrow$ \\ \hline
DP-FedAvg~\cite{brendan2018learning} & $41.43$ &          $29.28$                  &    $22.45$                      &      $13.50$                     &    $13.59$               &           $7.41$                 \\
DP-FedAvg (subsample)& 40.90 & 28.34 & 16.54	& 9.38& 13.24& 7.13        \\
+Ours & $70.59$ & {$67.72$} & {$63.28$} & {$61.12$} &  {$58.14$} & {$56.18$}\\
\bottomrule
\end{tabular}
\end{table}

\begin{table}[!tp]
\centering
\caption{Performance comparison with more training rounds.}
\label{app:tab:more_rounds}
\setlength{\tabcolsep}{5pt}
\begin{tabular}{l c c c c c c }
\toprule

\multirow{2}{*}{Method} & \multicolumn{2}{c}{$z=0.3$}  &\multicolumn{2}{c}{$z=0.5$} & \multicolumn{2}{c}{$z=0.7$} \\ 
 \cmidrule{2-7} & Dice $\uparrow$ & {IoU} $\uparrow$ & {Dice} $\uparrow$  & {IoU} $\uparrow$ & {Dice} $\uparrow$& {IoU} $\uparrow$ \\ \hline
DP-FedAvg~\cite{brendan2018learning} & $41.43$ &          $29.28$                  &    $22.45$                      &      $13.50$                     &    $13.59$               &           $7.41$                 \\
+Ours & $70.59$ & {$67.72$} & {$63.28$} & {$61.12$} &  {$58.14$} & {$56.18$}\\
\hline
DP-FedAvg (300 rounds)~\cite{brendan2018learning} & 57.65 & 44.17 & 30.04 & 19.49 & 18.84 & 10.92 \\
+Ours (300 rounds)  &  82.13	 & 73.86 & 75.98 & 65.84 & 76.96 & 68.41 \\
\bottomrule
\end{tabular}
\end{table}

\subsubsection{More training rounds}
We investigate more training rounds to validate DP accuracy deficits recovery. Specifically, we further extend training to 300 rounds, considering three privacy levels on the segmentation tasks, the results are shown in Table~\ref{app:tab:more_rounds}. From the table, it can be observed that, with more training rounds, both DP-FedAvg and our method have performance improvements, which indicates the effects of DP accuracy deficits recovery. However, our method still outperforms the DP-FedAvg, showing potential for overcoming the noisy gradients. Furthermore, it is noted that more training rounds will require larger privacy budgets, increasing overall privacy leakage risk, which is a trade-off between performance and privacy.

\subsubsection{Higher communication frequency}
We perform a comparison with more frequent aggregation to investigate how the performance-privacy trade-off changes. We have increased the communication frequency of DP-FedAvg by three times and present the results in Table~\ref{app:tab:high_freq}. From the results, we can see the Dice scores on prostate MRI data are 55.40, 31.46, and 32.22 for 3 privacy levels. Despite the increased frequency, our method still clearly outperforms the DP-FedAvg method. In addition, the increasing frequency will also incur a higher communication cost and risk of privacy leakage.

\begin{table}[!tp]
\centering
\caption{Performance comparison with higher communication frequency.}
\label{app:tab:high_freq}
\setlength{\tabcolsep}{5pt}
\begin{tabular}{l c c c c c c }
\toprule

\multirow{2}{*}{Method} & \multicolumn{2}{c}{$z=0.3$}  &\multicolumn{2}{c}{$z=0.5$} & \multicolumn{2}{c}{$z=0.7$} \\ 
 \cmidrule{2-7} & Dice $\uparrow$ & {IoU} $\uparrow$ & {Dice} $\uparrow$  & {IoU} $\uparrow$ & {Dice} $\uparrow$& {IoU} $\uparrow$ \\ \hline
DP-FedAvg~\cite{brendan2018learning} & $41.43$ &          $29.28$                  &    $22.45$                      &      $13.50$                     &    $13.59$               &           $7.41$                 \\
DP-FedAvg~\cite{brendan2018learning} (high frequency) & 55.40 & 42.26 & 31.46 & 20.50 & 32.22 & 21.12\\
+Ours & $70.59$ & {$67.72$} & {$63.28$} & {$61.12$} &  {$58.14$} & {$56.18$}\\

\bottomrule
\end{tabular}
\end{table}

\subsection{Privacy protection under attacks.}
To investigate the actual effects of privacy protection, we further deploy the gradient inversion attack on models, which aims to recover the original images. We adopt the average structural similarity (SSIM) between the original image and the attack reconstructed one as the metric. The results are shown in Table~\ref{app:tab:ssim}. We perform the attack on all prostate samples, and the average SSIMs are 5e-2, 1e-2, 7e-3 for DP-FedAvg with three DP levels, while the average SSIMs of our are 1e-2, 1e-2, 1e-2. Our method effectively against the attack while maintaining a higher performance at the same time.

\begin{table}[!tp]
\centering
\caption{Performance comparison with higher communication frequency.}
\label{app:tab:ssim}
\setlength{\tabcolsep}{5pt}
\begin{tabular}{l c c c }
\toprule

\multirow{1}{*}{Method} & \multicolumn{1}{c}{$z=0.3$}  &\multicolumn{1}{c}{$z=0.5$} & \multicolumn{1}{c}{$z=0.7$} \\ \hline
DP-FedAvg~\cite{brendan2018learning} &  5e-2 &  1e-2 &  7e-3          \\
+Ours & 1e-2 &  1e-2 &  1e-2\\
\bottomrule
\end{tabular}
\end{table}

\end{document}